\documentclass{llncs}
\usepackage{amsmath}
\usepackage{amssymb}
\usepackage{times}
\usepackage{indentfirst}
\usepackage{graphicx}

\begin{document}

\title{Relation matroid and its relationship with generalized rough set based on relation}         

\author{Yanfang Liu, William Zhu~\thanks{Corresponding author.
E-mail: williamfengzhu@gmail.com (William Zhu)}}
\institute{
Lab of Granular Computing,\\
Zhangzhou Normal University, Zhangzhou 363000, China}


\date{\today}          
\maketitle

\begin{abstract}
Recently, the relationship between matroids and generalized rough sets based on relations has been studied from the viewpoint of linear independence of matrices.
In this paper, we reveal more relationships by the predecessor and successor neighborhoods from relations.
First, through these two neighborhoods, we propose a pair of matroids, namely predecessor relation matroid and successor relation matroid, respectively.
Basic characteristics of this pair of matroids, such as dependent sets, circuits, the rank function and the closure operator, are described by the predecessor and successor neighborhoods from relations.
Second, we induce a relation from a matroid through the circuits of the matroid.
We prove that the induced relation is always an equivalence relation.
With these two inductions, a relation induces a relation matroid, and the relation matroid induces an equivalence relation, then the connection between the original relation and the induced equivalence relation is studied.
Moreover, the relationships between the upper approximation operator in generalized rough sets and the closure operator in matroids are investigated.

\end{abstract}

\section{Introduction}
Rough set theory was proposed by Pawlak~\cite{Pawlak91Rough} as a mathematical tool to organize and analyze various types of data in data mining.
It is especially useful for dealing with uncertain and vague knowledge in information systems.
The classical rough set theory is based on equivalence relations or partitions.
And the key notions of rough set theory, i.e., the lower and upper approximations are firstly constructed through equivalence classes.
However, the notion of equivalence relations or partitions is too restrictive for many applications.
In order to meet many application requirements, various extensions of equivalence relations or partitions have been proposed, such as similarity relations~\cite{Vakarelov05AModal}, tolerance relations~\cite{YaoWangZhang10Onfuzzy}, arbitrary binary relations~\cite{Yao98Relational}, coverings~\cite{Bryniarski89ACalculus} and others~\cite{LiLeungZhang08Generalizedfuzzyrough,LiZhang08fuzzyroughapproximations,DengChenXuDai07anovalapproach}.
Therefore, the classical rough set theory has been extended to generalized rough sets based on similarity relation~\cite{SlowinskiVanderpooten00AGeneralized}, generalized rough sets based on tolerance relation~\cite{MengShi09AFast,SkowronStepaniuk96tolerance}, generalized rough sets based on binary relation~\cite{Zhu07Generalized,Kondo05OnTheStructure} and covering-based rough sets~\cite{Zhu07Basic,WangZhuZhu10Structure,Zhu06PropertiesoftheSecond,Zhu09RelationshipAmong,ZhuWang03Reduction,QinGaoPei07OnCovering,LiuSai09AComparison,FengMiWu06Covering-Based,ChenWangHu07ANew}.
In this paper, we focus on generalized rough sets based on relations.

Matroid theory was proposed by Whitney~\cite{Lai01Matroid} as a generalization of linear independence in vector spaces.
It borrowed extensively from linear algebra and graph theory, and made great progress in recent decades.
Matroids have sound theoretical foundations and wide applications.
In theory, matroids are connected with rough sets in literatures~\cite{WangZhuMin11Transversal,WangZhuMin11TheVectorially,ZhuWang11Rough,WangZhu11Matroidal,ZhuWang11Matroidal,LiuZhuZhang12Relationshipbetween,ZhangWangFengFeng11reductionofrough}. In application, matroids have been used in diverse fields, such as greedy algorithm~\cite{RoyVoxman92FuzzyMatroid}, combinatorial optimization~\cite{Lawler01Combinatorialoptimization} and network flows~\cite{Edmonds71Matroids}.
In this paper, we build the connection between generalized rough sets based on relations and matroids through the predecessor and successor neighborhoods.
The main contributions of this work are 2-fold.

On the one hand, neighborhood is an important concept in generalized rough sets based on relations.
For an arbitrary binary relation, the predecessor and successor neighborhoods can be obtained.
Through these two neighborhoods, we propose a pair of matroids called predecessor relation matroid and successor relation matroid, respectively.
This pair of matroids can be expressed by each other, therefore we need to investigate only one of the two matroids.
In this paper, we mainly study the successor relation matroid, or relation matroid for brevity.
Basic characteristics of the relation matroid induced by a relation, such as dependent sets, circuits, the rank function and the closure operator, are described by the relation.
Moreover, the upper approximation operator of a relation in generalized rough sets based on relations is compared with the closure operator of the relation matroid induced by the relation.
Especially, the two operators are equal to each other when the relation is an equivalence relation.

On the other hand, we induce a relation from a matroid through the circuits of the matroid.
And we prove that the induced relation is always an equivalence relation.
Similarly, the closure operator of a matroid is compared with the upper approximation operator of the relation induced by the matroid.
Especially, we investigate the conditions when the closure operator of a matroid is equal to the upper approximation operator of the relation induced by the matroid.
Based on these two inductions, a relation induces a relation matroid, and the relation matroid induces an equivalence relation, we prove that the induced equivalence relation can be expressed by the original relation.
Moreover, we prove that the induced equivalence relation is equal to the original relation if and only if the original relation is an equivalence relation.

The rest of this paper is arranged as follows.
Section~\ref{S:preliminaries} reviews some definitions and results of Pawlak's rough sets, generalized rough sets based on relations and matroids.
In Section~\ref{S:matroidinducedbybinaryrelation}, we propose predecessor and successor relation matroids induced by binary relations, and study some characteristics of them through neighborhoods in generalized rough sets based on relations.
Section~\ref{S:relationinducedbymatroid} induces a relation from a matroid and investigates the relationship between the upper approximation operator and the closure operator.
And we conclude this paper in Section~\ref{S:conclusions}.

\section{Preliminaries}
\label{S:preliminaries}

In this section, we recall some basic definitions of binary relations and equivalence relations which is a special type of binary relations.
Then, we review some results of Pawlak's rough sets based on equivalence relations, generalized rough sets based on binary relations and matroids.

\subsection{Binary relation}

Let $U$ be a set and $U\times U$ the product set of $U$ and $U$.
If $R\in U\times U$, then $R$ is called a binary relation~\cite{RajagopalMason92Discrete} on $U$.
For all $(x, y)\in U\times U$, if $(x, y)\in R$, we say $x$ has relation $R$ with $y$, and denote this relationship as $xRy$.

In mathematics, the inverse relation of a binary relation is the relation that occurs when you switch the order of the elements in the relation.
Then, the inverse relation of a binary relation is introduced in the following definition.

\begin{definition}(Inverse relation\cite{RajagopalMason92Discrete})
Let $R$ be a relation on $U$.
Then,\\
\centerline{$R^{-1}=\{(y, x)\in U\times U:(x, y)\in R\}$,}
where $R^{-1}$ is the inverse relation of $R$.
\end{definition}

Throughout this paper, a binary relation is simply called a relation.
Reflective, symmetric, and transitive properties play important roles in characterizing relations.
Then, we introduce equivalence relations through these three properties.

\begin{definition}(Reflexive, symmetric and transitive~\cite{RajagopalMason92Discrete})
Let $R$ be a relation on $U$.\\
If for all $x\in U$, $xRx$, we say $R$ is reflexive.\\
If for all $x, y\in U$, $xRy$ implies $yRx$, we say $R$ is symmetric.\\
If for all $x, y, z\in U$, $xRy$ and $yRz$ imply $xRz$, we say $R$ is transitive.
\end{definition}

\begin{definition}(Equivalence relation~\cite{RajagopalMason92Discrete})
Let $R$ be a relation on $U$.
If $R$ is reflexive, symmetric and transitive, we say $R$ is an equivalence relation on $U$.
\end{definition}

\subsection{Pawlak's rough set}
\label{S:Pawlak'sroughset}

Let $U$ be a finite and nonempty set and $R$ an equivalence relation on $U$.
The equivalence relation $R$ induces a partition $U/R=\{[x]_{R}:x\in U\}$ on $U$, where $[x]_{R}=\{y\in U:xRy\}$ is the equivalence class determined by $x$ with respect to $R$.

In rough set theory~\cite{Pawlak91Rough}, the equivalence classes of $R$ are elementary sets to construct lower and upper approximations.
For any $X\subseteq U$, its lower and upper approximations are defined as follows:
\begin{center}
~$\underline{R}(X)=\{x\in U:[x]_{R}\subseteq X\}$;\\
\quad\quad $\overline{R}(X)=\{x\in U:[x]_{R}\bigcap X\neq\emptyset\}$.
\end{center}

$X^{c}$ is denoted by the complement of $X$ in $U$ and $Y\subseteq U$. We have the following properties of rough sets:\\
(1L) $\underline{R}(U)=U$;\\
(1H) $\overline{R}(U)=U$;\\
(2L) $\underline{R}(\emptyset)=\emptyset$;\\
(2H) $\overline{R}(\emptyset)=\emptyset$;\\
(3L) $\underline{R}(X)\subseteq X$;\\
(3H) $X\subseteq \overline{R}(X)$;\\
(4L) $\underline{R}(X\bigcap Y)=\underline{R}(X)\bigcap \underline{R}(Y)$;\\
(4H) $\overline{R}(X\bigcup Y)=\overline{R}(X)\bigcup \overline{R}(Y)$;\\
(5L) $\underline{R}(\underline{R}(X))=\underline{R}(X)$;\\
(5L) $\overline{R}(\overline{R}(X))=\overline{R}(X)$;\\
(6L) $X\subseteq Y\Rightarrow \underline{R}(X)\subseteq \underline{R}(X)$;\\
(6H) $X\subseteq Y\Rightarrow \overline{R}(X)\subseteq \overline{R}(X)$;\\
(7L) $\underline{R}(X^{c})=(\overline{R}(X))^{c}$;\\
(7H) $\overline{R}(X^{c})=(\underline{R}(X))^{c}$;\\
(8L) $\underline{R}((\underline{R}(X))^{c})=(\underline{R}(X))^{c}$;\\
(8H) $\overline{R}((\overline{R}(X))^{c})=(\overline{R}(X))^{c}$.

The (3L), (3H), (4L), (4H), (8L) and (8H) are characteristic properties of the lower and upper approximation operators~\cite{LinLiu94Rough,ZhuHe00TheAxiomization}, in other words, all other properties can be deduced from these properties.

\subsection{Generalized rough set based on binary relation}

Pawlak's rough sets are based on equivalence relations.
and the requirement of equivalence relations is a very restrictive condition that limit the application domain of the rough set model.
Therefore, many authors have generalized the notion of approximation operators by using more general binary relations or by employing coverings.
In this section, we introduce a type of generalized rough sets based on relations which extend Pawlak's rough sets through extending an equivalence relation to a binary relation.
First, we introduce the successor neighborhood and predecessor neighborhood of any element in a relation.

\begin{definition}(Successor neighborhood and predecessor neighborhood~\cite{ZhuWang07Topological})
Let $R$ be a relation on $U$.
For any $x\in U$,\\
\centerline{$RS_{R}(x)=\{y\in U:xRy\}$;}\\
\centerline{$RP_{R}(x)=\{y\in U:yRx\}$.}\\
We call the set $RS_{R}(x), RP_{R}(x)$ the successor neighborhood and the predecessor neighborhood of $x$ in $R$, respectively.
\end{definition}

The successor and predecessor neighborhoods are important concepts in generalized rough sets based on relations, and they are used to construct lower and upper approximation operators.
In this paper, we only use the successor neighborhood.
Therefore, the lower and upper approximation operators based on successor neighborhood are introduced as followed.

\begin{definition}(Lower and upper approximation operators~\cite{Zhu09RelationshipBetween})
\label{D:lowerandupperapproximation}
Let $R$ be a relation on $U$.
Lower and upper approximation operators $L_{R}, H_{R}:2^{U}\rightarrow 2^{U}$ are defined as follows:
for all $X\subseteq U$,
\begin{center}
$L_{R}(X)=\{x\in U:RS_{R}(x)\subseteq X\}$;\\
~~~~~$H_{R}(X)=\{x\in U:RS_{R}(x)\bigcap X\neq\emptyset\}$.
\end{center}
\end{definition}

Because of the duality, only the properties of the upper approximation operators are presented.

\begin{proposition}(\cite{Zhu07Generalized})
Let $R$ be a relation on $U$.
$H_{R}$ satisfies the following properties: for all $X, Y\subseteq U$,\\
(1) $H_{R}(\emptyset)=\emptyset$;\\
(2) $H_{R}(X\bigcup Y)=H_{R}(X)\bigcup H_{R}(Y)$;\\
(3) $X\subseteq Y\Rightarrow H_{R}(X)\subseteq H_{R}(Y)$.
\end{proposition}

\subsection{Matroid}

Matroids are combinatorial structures that generalize the notion of linear independence in matrices.
Matroids have many applications in various fields, partly because of a number of axiom systems.
The following definition through independent sets is widely used.

\begin{definition}(Matroid~\cite{Lai01Matroid})
\label{D:matroid}
A matroid is a pair $M=(U, \mathbf{I})$ consisting of a finite set $U$ and a collection $\mathbf{I}$ of subsets of $U$ called independent sets satisfying the following three properties:\\
(I1) $\emptyset\in\mathbf{I}$;\\
(I2) If $I\in\mathbf{I}$ and $I'\subseteq I$, then $I'\in\mathbf{I}$;\\
(I3) If $I_{1}, I_{2}\in\mathbf{I}$ and $|I_{1}|<|I_{2}|$, then there exists $u\in I_{2}-I_{1}$ such that $I_{1}\bigcup\{u\}\in\mathbf{I}$, where $|I|$ denotes the cardinality of $I$.
\end{definition}

In order to make some expressions brief, we introduce some symbols as follows.

\begin{definition}(\cite{Lai01Matroid})
Let $U$ be a finite set and $\mathbf{A}$ a family of subsets of $U$.
Then\\
$Min(\mathbf{A})=\{X\in\mathbf{A}:\forall Y\in\mathbf{A}, Y\subseteq X\Rightarrow X=Y\}$;\\
$Opp(\mathbf{A})=\{X\subseteq U:X\notin\mathbf{A}\}$.
\end{definition}

In a matroid, any subset of a set is not an independent set, then it is a dependent set, and vice versa.
In other words, the complement of the independent sets are dependent ones.
Then, the dependent sets of a matroid are represented in the following definition.

\begin{definition}(Dependent set~\cite{Lai01Matroid})
\label{D:dependentset}
Let $M=(U, \mathbf{I})$ be a matroid and $X\subseteq U$.
If $X\notin\mathbf{I}$, then $X$ is called a dependent set.
The family of all dependent sets of $M$ is denoted by $\mathbf{D}(M)$, where $\mathbf{D}(M)=Opp(\mathbf{I})$.
\end{definition}

Any minimal dependent set is called a circuit of a matroid.
A matroid uniquely determines its circuits, and vice versa.
The circuits of a matroid are represented as follows.

\begin{definition}(Circuit~\cite{Lai01Matroid})
\label{D:circuit}
Let $M=(U, \mathbf{I})$ be a matroid.
A minimal dependent set in $M$ is called a circuit of $M$, and the family of all circuits of $M$ is denoted by $\mathbf{C}(M)$, i.e., $\mathbf{C}(M)=Min(\mathbf{D}(M))$.
\end{definition}

In matroid theory, the rank function serves as a quantitative tool.
The cardinality of a maximal independent set of any subset can be expressed by the rank function.

\begin{definition}(Rank function~\cite{Lai01Matroid})
\label{D:rankfunction}
Let $M=(U, \mathbf{I})$ be a matroid.
Then $r_{M}$ is called the rank function of $M$, where $r_{M}(X)=max\{|I|:I\subseteq X, I\in\mathbf{I}\}$ for all $X\subseteq U$.
\end{definition}

We can define a matroid from the perspective of the rank function.
Then, the connection between a matroid and its rank function is introduced.

\begin{proposition}
\label{P:rankdetermineamatroid}
Let $M=(U, \mathbf{I})$ be a matroid and $r_{M}$ its rank function.
For all $X\subseteq U$, $r_{M}(X)=|X|$ if and only if $X\in\mathbf{I}$.
\end{proposition}

The closure operator of a matroid is defined through the rank function in a matroid.
And, the closure operator represents the relationship between an element and a subset of a set.

\begin{definition}(Closure operator~\cite{Lai01Matroid})
\label{D:closure}
Let $M=(U, \mathbf{I})$ be a matroid and $X\subseteq U$.
For any $u\in U$, if $r_{M}(X)=r_{M}(X\bigcup\{u\})$, then $u$ depends on $X$.
The subset including all elements depending on $X$ of $U$ is called the closure with respect to $X$ and denoted by $cl_{M}(X)$:\\
\centerline{$cl_{M}(X)=\{u\in U:r_{M}(X)=r_{M}(X\bigcup\{u\})\}$,}
where $cl_{M}$ is called the closure operator of $M$.
\end{definition}

A matroid can be defined from the viewpoint of the closure operator.
And the closure operator can uniquely determine a matroid with each other.

\begin{proposition}(Closure axioms~\cite{Lai01Matroid})
\label{P:closureaxioms}
Let $cl:2^{U}\rightarrow 2^{U}$ be a operator.
Then there exists a matroid $M$ such that $cl=cl_{M}$ iff $cl$ satisfies the following conditions:\\
(CL1) If $X\subseteq U$, then $X\subseteq cl(X)$;\\
(CL2) If $X\subseteq Y\subseteq U$, then $cl(X)\subseteq cl(Y)$;\\
(CL3) If $X\subseteq U$, then $cl(cl(X))=cl(X)$;\\
(CL4) If $x\in U$, $X\subseteq U$, and $y\in cl(X\bigcup\{x\})-cl(X)$, then $x\in cl(X\bigcup \{y\})$.
\end{proposition}

In a matroid, if the closure of a set is equal to itself, then the set is a closed set.

\begin{definition}(Closed set~\cite{Lai01Matroid})
Let $M=(U, \mathbf{I})$ be a matroid and $X\subseteq U$.
$X$ is a closed set of $M$ if $cl(X)=X$.
\end{definition}

\section{Matroid induced by a relation}
\label{S:matroidinducedbybinaryrelation}

In this section, we induce a pair of matroids by a relation.
Firstly, we define two set families through the successor and predecessor neighborhoods of a relation.

\begin{definition}
\label{D:defineIR}
Let $U$ be a nonempty finite set and $R$ a relation on $U$.
Then, we define two set families as follows:
\begin{center}
$\mathbf{I}_{S}(R)=\{X\subseteq U:\forall x, y\in X, x\neq y\Rightarrow RS_{R}(x)\neq RS_{R}(y)\}$;\\
$\mathbf{I}_{P}(R)=\{X\subseteq U:\forall x, y\in X, x\neq y\Rightarrow RP_{R}(x)\neq RP_{R}(y)\}$.
\end{center}
\end{definition}

These two set families satisfy the independent set properties of Definition~\ref{D:matroid} as shown in the following proposition.

\begin{proposition}
\label{P:IRsatisfiesindependentsets}
Let $R$ be a relation on $U$.
Then $\mathbf{I}_{S}(R)$ and $\mathbf{I}_{P}(R)$ satisfy (I1), (I2) and (I3), respectively.
\end{proposition}

\begin{proof}
We first prove $\mathbf{I}_{S}(R)$ satisfies (I1), (I2) and (I3).

(I1): $\emptyset\in\mathbf{I}_{S}(R)$ is straightforward.

(I2): If $I\in\mathbf{I}_{S}(R)$ and $I'\subseteq I$, then $I'\in\mathbf{I}_{S}(R)$.

Suppose that $I'\notin\mathbf{I}_{S}(R)$, then there exists $x, y\in I'$ such that $RS_{R}(x)=RS_{R}(y)$.
$I'\subseteq I$, then $x, y\in I$ such that $RS_{R}(x)=RS_{R}(y)$, this is contradictory to $I\in\mathbf{I}_{S}(R)$.
Therefore $I'\in\mathbf{I}_{S}(R)$.

(I3): If $I_{1}, I_{2}\in\mathbf{I}_{S}(R)$ and $|I_{1}|<|I_{2}|$, then there exists $u\in U$ such that $I_{1}\bigcup\{u\}\in\mathbf{I}_{S}(R)$.

Since $I_{1}, I_{2}\in\mathbf{I}_{S}(R)$, then for all $x_{1}, y_{1}\in I_{1}, x_{1}\neq y_{1}$ and $x_{2}, y_{2}\in I_{2}, x_{2}\neq y_{2}$, $RS_{R}(x_{1})\neq RS_{R}(y_{1}), RS_{R}(x_{2})\neq RS_{R}(y_{2})$. 
Suppose that for all $u\in I_{2}-I_{1}$, $I_{1}\bigcup\{u\}\notin\mathbf{I}_{S}(R)$, then there exists one and only one $x\in I_{1}-I_{2}$ such that $RS_{R}(u)=RS_{R}(x)$.
Therefore, $|I_{2}-I_{1}| = |I_{1}-I_{2}|$ which is contradictory to $|I_{1}|<|I_{2}|$.
Hence, there exists $u\in I_{2}-I_{1}$ such that $I_{1}\bigcup\{u\}\in\mathbf{I}_{S}(R)$.

Similar to the proof of $\mathbf{I}_{S}(R)$, it is easy to prove that $\mathbf{I}_{P}(R)$ satisfies (I1), (I2) and (I3).
\end{proof}

As is shown in the above proposition, $\mathbf{I}_{P}(R)$ and $\mathbf{I}_{S}(R)$ are independent sets, so they can generate two matroids, respectively.

\begin{definition}(Successor relation matroid and predecessor relation matroid)
\label{D:twomatroids}
Let $R$ be a relation on $U$.
The one matroid with $\mathbf{I}_{S}(R)$ as its independent sets is denoted by $M_{S}(R)=(U, \mathbf{I}_{S}(R))$.
And the other matroid with $\mathbf{I}_{P}(R)$ as its independent sets is denoted by $M_{P}(R)=(U, \mathbf{I}_{P}(R))$.
We say $M_{S}(R)$ and $M_{P}(R)$ the successor relation matroid and predecessor relation matroid induced by $R$, respectively.
\end{definition}

We illustrate the successor and predecessor relation matroids induced by a relation with the following example.

\begin{example}
\label{E:tworelationmatroids}
Let $U=\{1, 2, 3\}$ and $R=\{(1,1), (1, 2), (2, 1), (2, 3), (3, 1), (3, 3)\}$.
Then $RS_{R}(1)=\{1, 2\}, RS_{R}(2)=\{1, 3\}, RS_{R}(3)=\{1, 3\}$ and $RP_{R}(1)=\{1, 2, 3\},$ $ RP_{R}(2)=\{1\}, RP_{R}(3)=\{2, 3\}$.
Therefore $M_{S}(R)=(U, \mathbf{I}_{S}(R))$ is the successor relation matroid induced by $R$, where $\mathbf{I}_{S}(R)=\{\emptyset, \{1\}, \{2\}, \{3\}, \{1, 2\}, \{1, 3\}\}$.
$M_{P}(R)=(U, \mathbf{I}_{P}(R))$ is the predecessor relation matroid induced by $R$, where $\mathbf{I}_{P}(R)=\{\emptyset, \{1\}, \{2\}, \{3\}, \{1, 2\}, \{1, 3\}, \{2, 3\}, \{1, 2, 3\}\}$.
\end{example}

We study the relationship between successor relation matroids and predecessor relation matroids.
First, we introduce a lemma in the following.

\begin{lemma}
\label{L:successorandpredecessor}
Let $R$ be a relation on $U$.
Then for any $x\in U$,\\
\centerline{$RS_{R}(x)=RP_{R^{-1}}(x)$.}
\end{lemma}

We see, with regard to every relation on a universe, there exists a relation to be the inverse relation of the relation.
For any element in a universe, its successor neighborhood of a relation is equal to its predecessor neighborhood of the inverse relation of the relation.
Therefore, the relationship between the successor and predecessor relation matroids induced by a relation is studied in the following.

\begin{theorem}
\label{T:relationshipbetweentwomatroids}
Let $R$ be a relation on $U$.
Then $M_{S}(R)=M_{P}(R^{-1})$.
\end{theorem}

\begin{proof}
According to Definition~\ref{D:twomatroids} and Lemma~\ref{L:successorandpredecessor}, it is straightforward.
\end{proof}

According to Theorem~\ref{T:relationshipbetweentwomatroids}, the successor relation matroid induced by a relation can be equivalently expressed by the predecessor relation matroid induced by the inverse relation of the relation.
Therefore, in this paper, we need to consider only one of the successor and predecessor relation matroids induced by a relation.
Unless otherwise stated, in the following we study the successor relation matroid induced by a relation, for short, relation matroid. $\mathbf{I}_{S}(R)$ and $M_{S}(R)$ are denoted by $\mathbf{I}(R), M(R)$ for brevity, respectively.

In a matroid, if a subset of the universe is not an independent set, then it is a dependent one.
What are characteristics of dependent sets of a relation matroid?

\begin{proposition}
\label{P:DRisdependentset}
Let $R$ be a relation on $U$ and $M(R)=(U, \mathbf{I}(R))$ the relation matroid induced by $R$.
Then,
\begin{center}
$\mathbf{D}(M(R))=\{X\subseteq U:\exists x, y\in X, x\neq y$ s.t. $RS_{R}(x)=RS_{R}(y)\}$.
\end{center}
\end{proposition}

\begin{proof}
According to Definition~\ref{D:dependentset}, we only need to prove $\mathbf{D}(M(R))=Opp(\mathbf{I}(R))$.
For all $X\in\mathbf{D}(M(R))$, there exist $x, y\in X$ and $x\neq y$ such that $RS_{R}(x)=RS_{R}(y)$.
Then, $X\notin\{X\subseteq U:\forall x, y\in X, x\neq y\Rightarrow RS_{R}(x)\neq RS_{R}(y)\}=\mathbf{I}(R)$, i.e., $X\in Opp(\mathbf{I}(R))$, i.e., $\mathbf{D}(M(R))\subseteq Opp(\mathbf{I}(R))$.
Conversely, for all $X\in Opp(\mathbf{I}(R))$, i.e., $X\notin\mathbf{I}(R)=\{X\subseteq U:\forall x, y\in X, x\neq y\Rightarrow RS_{R}(x)\neq RS_{R}(y)\}$, then there exist $x, y\in X$ and $x\neq y$ such that $RS_{R}(x)=RS_{R}(y)$, i.e., $X\in\mathbf{D}(M(R))$, i.e., $Opp(\mathbf{I}(R))\subseteq\mathbf{D}(M(R))$.
To sum up, this completes the proof.
\end{proof}

In a matroid, a minimal dependent set is a circuit.
We will study characteristics of the circuits of a relation matroid in the following proposition.

\begin{proposition}
\label{P:circuitofrelationmatroid}
Let $R$ be a relation on $U$ and $M(R)$ the relation matroid induced by $R$.
Then $\mathbf{C}(M(R))=\{\{x, y\}:x, y\in U\bigwedge x\neq y\bigwedge RS_{R}(x)=RS_{R}(y)\}$.
\end{proposition}

\begin{proof}
According to Definition~\ref{D:circuit} and Proposition~\ref{P:DRisdependentset}, $\mathbf{C}(M(R))=Min(\mathbf{D}(M(R)))$ and $\mathbf{D}(M(R))=\{X\subseteq U:\exists x, y\in X, x\neq y$ s.t. $RS_{R}(x)=RS_{R}(y)\}$.
Then $\mathbf{C}(M(R))=\{\{x, y\}:x, y\in U\bigwedge x\neq y\bigwedge RS_{R}(x)=RS_{R}(y)\}$ is straightforward.
\end{proof}

The rank function is one of important characteristics in matroid theory.
In the following proposition, we will investigate the rank function of a relation matroid.

\begin{proposition}
\label{P:rankofrelationmatroid}
Let $R$ be a relation on $U$ and $M(R)$ the relation matroid induced by $R$.
Then for all $X\subseteq U$,
\begin{center}
$r_{M(R)}(X)=|\{RS_{R}(x):x\in X\}|$.
\end{center}
\end{proposition}

\begin{proof}
According to Proposition~\ref{P:rankdetermineamatroid}, we only need to prove $r_{M(R)}(X)=|X|$ if and only if $X\in\mathbf{I}(R)$.
When $r_{M(R)}(X)=|X|$, then $|\{RS_{R}(x):x\in X\}|=|X|\Leftrightarrow\forall x\in X$, there exists $y\in X$, and $x\neq y$ such that $RS_{R}(x)\neq RS_{R}(y)$.
According to Definition~\ref{D:defineIR}, $X\in\mathbf{I}(R)$.
Conversely, $X\in\mathbf{I}(R)$, according to Definition~\ref{D:rankfunction}, $r_{M(R)}(X)=max\{|I|:I\subseteq X, I\in\mathbf{I}(R)\}$, therefore $r_{M(R)}(X)=|X|$.
To sum up, this completes the proof.
\end{proof}

According to Definition~\ref{D:closure}, the closure of a subset is a set of all elements depending on the subset in matroids.
In other words, the closure of a subset is all those elements when added to the subset, the rank is the same.
We will study the closure operator of  a relation matroid in the following proposition.

\begin{proposition}
\label{P:closureofrelationmatroid}
Let $R$ be a relation on $U$ and $M(R)$ the relation matroid induced by $R$.
Then for all $X\subseteq U$,
\begin{center}
$cl_{M(R)}(X)=\{u\in U:\exists x\in X, RS_{R}(x)=RS_{R}(u)\}$.
\end{center}
\end{proposition}

\begin{proof}
According to Definition~\ref{D:closure}, we only need to prove $\{u\in U:\exists x\in X, RS_{R}(x)=RS_{R}(u)\}=\{u\in U: r_{M(R)}(X\bigcup \{u\})=r_{M(R)}(X)\}$.\\
$\{u\in U:\exists x\in X, RS_{R}(x)=RS_{R}(u)\}=\{u\in U: r_{M(R)}(X\bigcup \{u\})=r_{M(R)}(X)\}$ $\Leftrightarrow X\bigcup\{u\in X^{c}:\exists x\in X, RS_{R}(x)=RS_{R}(u)\}=X\bigcup\{u\in X^{c}: r_{M(R)}(X\bigcup \{u\})$ $=r_{M(R)}(X)\}\Leftrightarrow\{u\in X^{c}:\exists x\in X, RS_{R}(x)=RS_{R}(u)\}=\{u\in X^{c}: r_{M(R)}(X\bigcup \{u\})=r_{M(R)}(X)\}$.
According to Proposition~\ref{P:rankofrelationmatroid}, for any $u\in X^{c}, r_{M(R)}($ $X\bigcup \{u\})=r_{M(R)}(X)\Leftrightarrow |\{RS_{R}(x):x\in X\bigcup\{u\}\}|=|\{RS_{R}(x):x\in X\}|\Leftrightarrow\exists x\in X,$ s.t. $RS_{R}(x)=RS_{R}(u)$.
To sum up, this completes the proof.
\end{proof}

A set is called a closed set in a matroid if its closure is equal to the set itself.
In the following, we investigate closed sets of a relation matroid.

\begin{proposition}
Let $R$ be a relation on $U$ and $M(R)$ the relation matroid induced by $R$.
For all $X\subseteq U$, $cl_{M(R)}(X)=X$ if and only if $RS_{R}(x)\neq RS_{R}(u)$ where $x\in X$ and $u\in X^{c}$.
\end{proposition}

\begin{proof}
According to Proposition~\ref{P:closureofrelationmatroid}, $cl_{M(R)}(X)=\{u\in U:\exists x\in X, RS_{R}(x)=RS_{R}(u)\}$.
Then, $cl_{M(R)}(X)=X\Leftrightarrow\{u\in X^{c}:\exists x\in X, RS_{R}(x)=RS_{R}(u)\}=\emptyset\Leftrightarrow\forall u\in X^{c}, x\in X, RS_{R}(x)\neq RS_{R}(u)$.
\end{proof}

According to Proposition~\ref{P:IRsatisfiesindependentsets}, any relation can generate a relation matroid.
However, can different relations on a set generate different relation matroids or the same relation matroid?
If different relations generate the same relation matroid, what is the relationship with the relations?
We study these issues as follows.

\begin{theorem}
Let $R_{1}, R_{2}$ be two relations on $U$ and $M(R_{1}), M(R_{2})$ the relation matroids induced by $R_{1}, R_{2}$, respectively.
Then, $M(R_{1})=M(R_{2})$ if and only if $RS_{R_{1}}(x)=RS_{R_{1}}(y)\Leftrightarrow RS_{R_{2}}(x)=RS_{R_{2}}(y)$ for all $x, y\in U$.
\end{theorem}

\begin{proof}
According to Definition~\ref{D:defineIR} and Proposition~\ref{P:IRsatisfiesindependentsets}, it is straightforward.
\end{proof}

\begin{corollary}
Let $R_{1}, R_{2}$ be two relations on $U$ and $M(R_{1}), M(R_{2})$ the relation matroids induced by $R_{1}, R_{2}$, respectively.
Then, $M(R_{1})=M(R_{2})$ if and only if $RS_{R_{1}}(x)\neq RS_{R_{1}}(y)\Leftrightarrow RS_{R_{2}}(x)\neq RS_{R_{2}}(y)$ for all $x, y\in U$.
\end{corollary}

In rough sets, lower and upper approximations are important concepts.
Through them, an element of a set is judged to definitely, possibly, or impossibly belong to a subset of the set with respect to the knowledge on the set.
In the following, we investigate the relationships between the upper approximation operator with respect to a relation and the closure operator of the relation matroid induced by the relation.

\begin{proposition}
Let $R$ be a relation on $U$ and $M(R)$ the relation matroid induced by $R$.
For all $X\subseteq U$, if $R$ is reflexive, then $cl_{M(R)}(X)\subseteq H_{R}(X)$.
\end{proposition}

\begin{proof}
According to Definition~\ref{D:lowerandupperapproximation} and Proposition~\ref{P:closureofrelationmatroid}, $H_{R}(X)=\{x\in U:RS_{R}(x)\bigcap X\neq\emptyset\}$ and $cl_{M(R)}(X)=\{u\in U:\exists x\in X, RS_{R}(x)=RS_{R}(u)\}$.
Therefore, we only need to prove $\{u\in U:\exists x\in X, RS_{R}(x)=RS_{R}(u)\}\subseteq\{x\in U:RS_{R}(x)\bigcap X\neq\emptyset\}$.
$R$ is reflexive, then $x\in RS_{R}(x)$ for all $x\in U$.
For all $u\in\{u\in U:\exists x\in X, RS_{R}(x)=RS_{R}(u)\}$, there exists $x\in X$ such that $RS_{R}(x)=RS_{R}(u)$, then $x\in RS_{R}(u)$.
Therefore, $RS_{R}(u)\bigcap X\neq\emptyset$, i.e., $u\in\{x\in U:RS_{R}(x)\bigcap X\neq\emptyset\}$.
So $R$ is reflexive, then $cl_{M(R)}(X)\subseteq H_{R}(X)$.
\end{proof}

When a relation is reflexive, the upper approximation of a subset contains the closure of the subset in the relation matroid induced by the relation.
We will study the conditions that the upper approximation operator and the closure operator are equal to each other.
First, we introduce a lemma.

\begin{lemma}(\cite{Zhu07Generalized})
\label{L:upperandrelation}
Let $R$ be a relation on $U$.
Then,\\
$R$ is reflexive iff $X\subseteq H_{R}(X)$ for any $X\subseteq U$.\\
$R$ is transitive iff $H_{R}(H_{R}(X))\subseteq H_{R}(X)$ for any $X\subseteq U$.
\end{lemma}

\begin{theorem}
\label{T:upperandclosureindcedbyrelation}
Let $R$ be a relation on $U$ and $M(R)$ the relation matroid induced by $R$.
For all $X\subseteq U$, $cl_{M(R)}(X)=H_{R}(X)$ if and only if $R$ is an equivalence relation.
\end{theorem}

\begin{proof}
($\Rightarrow$): According to (CL1) of Proposition~\ref{P:closureaxioms}, $X\subseteq cl_{M(R)}(X)$.
And $cl_{M(R)}(X)=H_{R}(X)$, then $X\subseteq H_{R}(X)$.
According to Lemma~\ref{L:upperandrelation}, $R$ is reflexive.\\
For all $X\subseteq U$, $cl_{M(R)}(X)=H_{R}(X)$, then $cl_{M(R)}=H_{R}$.
According to (CL3) of Proposition~\ref{P:closureaxioms}, $cl_{M(R)}(cl_{M(R)}(X))=cl_{M(R)}(X)$, i.e., $H_{R}(H_{R}(X))=H_{R}(X)$.
According to Lemma~\ref{L:upperandrelation}, $R$ is transitive.\\
For all $y\in H_{R}(\{x\})=\{u\in U:RS_{R}(u)\bigcap\{x\}\neq\emptyset\}$, i.e., $RS_{R}(y)\bigcap\{x\}\neq\emptyset$, then $x\in RS_{R}(y)$, i.e., $(y, x)\in R$.
Therefore $y\in H_{R}(\{x\})\Leftrightarrow (y, x)\in R$.
According to Proposition~\ref{P:IRsatisfiesindependentsets}, for any $x\in U$, $\{x\}\in\mathbf{I}(R)$, then $cl_{M(R)}(\emptyset)=\emptyset$.
According to (CL4) of Proposition~\ref{P:closureaxioms}, $y\in H_{R}(\{x\})=cl_{M(R)}(\{x\})=cl_{M(R)}(\emptyset\bigcup\{x\})-cl_{M(R)}(\emptyset)$, then $x\in cl_{M(R)}(\emptyset\bigcup\{y\})=cl_{M(R)}(\{y\})=H_{R}(\{y\})$, i.e., $y\in H_{R}(\{x\})\Rightarrow x\in H_{R}(\{y\})$.
Therefore, $(y, x)\in R\Rightarrow (x, y)\in R$, then $R$ is symmetric.\\
Therefore, if $cl_{M(R)}(X)=H_{R}(X)$, then $R$ is an equivalence relation.\\
($\Leftarrow$): $R$ is an equivalence relation, then generalized rough sets based on relations are degenerated to Pawlak's rough sets.
According to Proposition~\ref{P:closureofrelationmatroid} and Definition~\ref{D:lowerandupperapproximation}, $cl_{M(R)}(X)=\{u\in U:\exists x\in X, RS_{R}(x)=RS_{R}(u)\}$ and $H_{R}(X)=\{x\in U:RS_{R}(x)\bigcap X\neq\emptyset\}$.
For any $u\in cl_{M(R)}(X)$, there exists $x\in X$ such that $x\in RS_{R}(x)=RS_{R}(u)$, then $RS_{R}(u)\bigcap X\neq\emptyset$, i.e., $cl_{M(R)}(X)\subseteq H_{R}(X)$.
Similarly, we can prove $H_{R}(X)\subseteq cl_{M(R)}(X)$.
\end{proof}

\section{Relation induced by a matroid}
\label{S:relationinducedbymatroid}

In Section~\ref{S:matroidinducedbybinaryrelation}, we have discussed a relation induces a matroid.
In this section, we induce a relation from a matroid.

\begin{definition}(\cite{Lai01Matroid})
\label{D:matroidinducerelation}
Let $M=(U, \mathbf{I})$ be a matroid.
We define a relation $R(M)$ on $U$ as follows:
for $x, y\in U$,
\begin{center}
$xR(M)y\Leftrightarrow x=y$ or $\{x, y\}\in\mathbf{C}(M)$.
\end{center}
We say $R(M)$ is a relation on $U$ induced by $M$.
\end{definition}

An example is provided to illustrate that how a matroid induces a relation in the following.

\begin{example}
Let $M=(U, \mathbf{I})$ be a matroid, where $U=\{1, 2, 3\}$ and $\mathbf{I}=\{\emptyset, \{1\}, \{3\}\}$.
Since $\mathbf{C}(M)=\{\{1, 3\}, \{2\}\}$, then $R(M)=\{(1, 1), (2, 2), (3, 3), (1, 3), (3, 1)\}$.
\end{example}

According to Definition~\ref{D:matroidinducerelation}, any matroid can induce a relation.
In fact, the relation induced by a matroid is an equivalence relation.

\begin{proposition}(\cite{Lai01Matroid})
\label{P:equivalencerelation}
Let $M=(U, \mathbf{I})$ be a matroid.
Then $R(M)$ is an equivalence relation.
\end{proposition}

The relationship between the upper approximation of any element in the relation induced by a matroid and the closure of the element in the matroid is studied.

\begin{proposition}
\label{P:oneelement}
Let $M=(U, \mathbf{I})$ be a matroid and $R(M)$ the relation induced by $M$.
Then, $H_{R(M)}(\{x\})\subseteq cl_{M}(\{x\})$ for all $x\in U$.
\end{proposition}

\begin{proof}
For all $y\in H_{R(M)}(\{x\})=\{y\in U:RS_{R(M)}(y)\bigcap\{x\}\neq\emptyset\}$, where $RS_{R(M)}(y)=\{z:yR(M)z\}$.
Therefore, $x\in RS_{R(M)}(y)$, i.e., $yR(M)x$.
According to Definition~\ref{D:matroidinducerelation}, $yR(M)x\Rightarrow\{y, x\}\in\mathbf{C}(M)$, then $y\in cl_{M}(\{x\})$.
Therefore, $H_{R(M)}(\{x\})\subseteq cl_{M}(\{x\})$ for all $x\in U$.
\end{proof}

According to Proposition~\ref{P:equivalencerelation}, the relation induced by a matroid is an equivalence relation.
The closure of an element in a matroid contains the upper approximation of the element with respect to the relation induced by the matroid.
We will study the conditions under which they are equal to each other.
In order to solve this issue, we introduce a lemma.

\begin{lemma}(\cite{Lai01Matroid})
\label{L:closureandcircuit}
Let $M=(U, \mathbf{I})$ be a matroid.
Then, $cl_{M}(X)=X\bigcup\{u\in U:\exists C\in\mathbf{C}(M),$ s.t. $u\in C\subseteq X\bigcup\{u\}\}$ for all $X\subseteq U$.
\end{lemma}

\begin{proposition}
Let $M=(U, \mathbf{I})$ be a matroid and $R(M)$ the relation induced by $M$.
Then, $cl_{M}(\{x\})=H_{R(M)}(\{x\})\bigcup\{y\in U:\{y\}\in\mathbf{C}(M)\}$ for all $x\in U$.
\end{proposition}

\begin{proof}
According to Lemma~\ref{L:closureandcircuit}, for all $\{x\}\subseteq U$, $cl_{M}(\{x\})=\{x\}\bigcup\{u\in U:\exists C\in\mathbf{C}(M),$ s.t. $u\in C\subseteq \{x, u\}\}=\{x\}\bigcup\{u\in U:\{x, u\}\in\mathbf{C}(M)\}\bigcup\{y\in U:\{y\}\in\mathbf{C}(M)\}$.
And for any $u\in H_{R(M)}(\{x\})=\{u\in U:RS_{R(M)}(u)\bigcap\{x\}\neq\emptyset\}\Leftrightarrow x\in RS_{R(M)}(u)\Leftrightarrow uR(M)x\Leftrightarrow x=u$ or $\{x, u\}\in\mathbf{C}(M)$, i.e., $H_{R(M)}(\{x\})=\{x\}\bigcup\{u\in U:\{x, u\}\in\mathbf{C}(M)\}$.
Therefore, $cl_{M}(\{x\})=H_{R(M)}(\{x\})\bigcup\{y\in U:\{y\}\in\mathbf{C}(M)\}$.
\end{proof}

According to Proposition~\ref{P:oneelement}, the relationship between the upper approximation and the closure of any subset is established.

\begin{corollary}
Let $M=(U, \mathbf{I})$ be a matroid and $X\subseteq U$.
Then $H_{R(M)}(X)\subseteq cl_{M}(X)$.
\end{corollary}

\begin{proof}
According to Proposition~\ref{P:oneelement}, $H_{R(M)}(X)=H_{R(M)}(\bigcup_{x\in X}\{x\})=\bigcup_{x\in X}H_{R(M)}$ $(\{x\})\subseteq\bigcup_{x\in X}cl_{M}$ $(\{x\})\subseteq cl_{M}(X)$.
\end{proof}

Based on the two inductions, a relation induces a relation matroid, and the relation matroid induces an equivalence relation, the connection between the original relation and the induced equivalence relation is built in the following theorem.

\begin{theorem}
\label{T:tworelations}
Let $R$ be a relation on $U$, $M(R)$ the relation matroid induced by $R$ and $R(M(R))$ the relation induced by $M(R)$.
Then $R(M(R))=\{(x, y)\in U\times U:RS_{R}(x)=RS_{R}(y)\}$.
\end{theorem}

\begin{proof}
According to Proposition~\ref{P:circuitofrelationmatroid}, $\mathbf{C}(M(R))=\{\{x, y\}:\forall x, y\in U, x\neq y\Rightarrow RS_{R}(x)=RS_{R}(y)\}$.
According to Definition~\ref{D:matroidinducerelation}, if $x\neq y$ and $\{x, y\}\in\mathbf{C}(M(R))$, then $xR(M(R))y$, i.e., $x\neq y$ and $RS_{R}(x)=RS_{R}(y)$, then $xR(M(R))y$.
If $x=y$, i.e., $RS_{R}(x)=RS_{R}(y)$, then $xR(M(R))y$.
Therefore, if $RS_{R}(x)=RS_{R}(y)$, then $xR(M(R))y$, i.e., $R(M(R))=\{(x, y)\in U\times U:RS_{R}(x)=RS_{R}(y)\}$.
\end{proof}

We illustrate Theorem~\ref{T:tworelations} with the following example.

\begin{example}
As shown in Example~\ref{E:tworelationmatroids}, $RS_{R}(1)=\{1, 2\}, RS_{R}(2)=\{1, 3\}, RS_{R}(3)=\{1, 3\}$.
Then, $RS_{R}(1)=RS_{R}(1), RS_{R}(2)=RS_{R}(2), RS_{R}(3)=RS_{R}(3), RS_{R}(2)=RS_{R}(3)$.
Therefore, $R(M(R))=\{(1, 1), (2, 2), (3, 3), (2, 3), (3, 2)\}$.
\end{example}

As is shown in the above theorem, the induced equivalence relation can be expressed by the original relation.
What is the relationship between the original relation and the induced equivalence relation?

\begin{proposition}
Let $R$ be a relation on $U$, $M(R)$ the relation matroid induced by $R$ and $R(M(R))$ the relation induced by $M(R)$.
If $R$ is reflexive, then $R(M(R))\subseteq R$.
\end{proposition}

\begin{proof}
According to Theorem~\ref{T:tworelations}, $R(M(R))=\{(x, y)\in U\times U:RS_{R}(x)=RS_{R}(y)\}$.
$R$ is reflexive, $x\in RS_{R}(x)$ for all $x\in U$.
For all $(x, y)\in R(M(R))$, $RS_{R}(x)=RS_{R}(y)$.
$R$ is reflexive, then $y\in RS_{R}(x)$, i.e., $(x,y)\in R$.
Therefore, $R(M(R))\subseteq R$.
\end{proof}

We see the relation induced by a matroid is an equivalence relation.
When a relation is an equivalence relation, what is the relationship between it and the relation induced by its induced relation matroid?

\begin{proposition}
Let $R$ be a relation on $U$, $M(R)$ the relation matroid induced by $R$ and $R(M(R))$ the relation induced by $M(R)$.
Then $R(M(R))=R$ if and only if $R$ is an equivalence relation.
\end{proposition}

\begin{proof}
According to Proposition~\ref{P:equivalencerelation} and Theorem~\ref{T:tworelations}, it is straightforward.
\end{proof}

\section{Conclusions}
\label{S:conclusions}

In this paper, we proposed a pair of matroids called predecessor and successor relation matroids through neighborhoods from generalized rough sets based on relations.
Predecessor relation matroids and successor relation matroids can be expressed by each other, so we only investigated successor relation matroids.
Successor relation matroid is called relation matroid for brevity.
We investigate some characteristics of the relation matroid induced by a relation, such as dependent sets, circuits, the rank function and the closure operator.
And the relationship between the upper approximation operator of a relation and the closure operator of the relation matroid induced by the relation is studied.
We also induced a relation from a matroid, and proved that the induced relation was always an equivalence relation.
The closure operator of a matroid was compared with the upper approximation operator of the relation induced by the matroid.
Moreover, based on these two inductions, a relation induced a relation matroid, and the relation matroid induced an equivalence relation, we  proved that the induced equivalence relation can be expressed by the original relation.
This study provided a simple however effective connection between generalized rough sets based on relations and matroids.

\section{Acknowledgments}
This work is in part supported by National Science Foundation of China under Grant No. 60873077, 61170128 and the Natural Science Foundation of
Fujian Province, China under Grant No. 2011J01374.



\end{document}